\documentclass[10pt,letterpaper]{article}

\usepackage[
  shownumpages,
  citingstyle=authoryear,
  bibliostyle=plainnat,
  bibfile=refs,
  boxrule=0.7pt
]{brainlab}
\usepackage{placeins}

\usepackage{titlesec}
\titlespacing*{\paragraph}{0pt}{0pt}{0.5em}

\hypersetup{
  urlcolor=braincolor,
  pdftitle={Monarch--Muon: Structure-Aware Muon for Fast and Memory-Efficient Pretraining},
  pdfauthor={Alexander Miasnikov, Kabirov Ruslan, Gleb Molodtsov, Dmitrii Vlashchuk, Vladimir Aletov, Nail Bashirov, Aleksandr Shestakov, Daniil Tikhonov, Maxim Krutikov, Aleksandr Beznosikov}
}

\graphicspath{{figures/}}
\providecommand{\bdiag}{\operatorname{bdiag}}
\providecommand{\op}{\mathrm{op}}

\newcounter{algorithmline}
\newcommand{\mmline}{%
  \refstepcounter{algorithmline}%
  \item[\scriptsize\thealgorithmline:]%
}

\tcbset{brainstyle/.append style={breakable}}

\makeatletter
\@addtoreset{theorem}{section}

\let\c@lemma\c@theorem
\let\c@corollary\c@theorem
\let\c@assumption\c@theorem

\makeatother

\renewcommand{\brainprintbibliography}{%
  \clearpage
  \begingroup
  \small
  \setlength{\bibsep}{2pt plus 1pt}
  \bibliographystyle{plainnat}
  \bibliography{refs}
  \endgroup
}

\setbrainmeta{
  title={Monarch--Muon: Structure-Aware Muon for Fast and Memory-Efficient Pretraining},
  authors={%
    \normalfont\mdseries\normalsize
    Alexander Miasnikov\textsuperscript{1},
    Kabirov Ruslan\textsuperscript{1},
    Gleb Molodtsov\textsuperscript{1},
    Dmitrii Vlashchuk\textsuperscript{1},
    Vladimir Aletov\textsuperscript{1},
    Nail Bashirov\textsuperscript{1},
    Daniil Tikhonov\textsuperscript{1,2},
    Maxim Krutikov\textsuperscript{1},
    Aleksandr Shestakov\textsuperscript{1,3,*},
    Aleksandr Beznosikov\textsuperscript{1}
  },
  affiliations={%
    {\small
      \textsuperscript{1}BRAIn Lab \quad
      \textsuperscript{2}HSE University \quad
      \textsuperscript{3}MBZUAI\\
      \textsuperscript{*}Work performed at BRAIn Lab prior to joining MBZUAI.
    }
  },
  code = \url{https://github.com/brain-lab-research/monarch-muon},
  abstract={Training large language models is constrained by both wall-clock time and memory footprint, and recent matrix-aware optimizers sharpen this tradeoff: Muon can improve optimization but its dense polar update remains expensive. We propose Monarch--Muon, a structured matrix-aware optimizer for weights parameterized by Monarch factors. The resulting structure reduces weights, gradients, and optimizer state, and exposes small independent update subproblems in place of a single dense orthogonalization; we formalize this through a blockwise trust-region interpretation. In quality runs, Monarch--Muon improves over Monarch--AdamW at fixed structure: 3.06 vs. 3.09 in the dense 257M model and 2.68 vs. 2.77 in the 1B MoE model; the two-block MoE variant also matches dense AdamW while using 15\% fewer active and 25.4\% fewer total parameters. In 6.89B-parameter systems measurements, Monarch--Muon reduces peak memory from 68.9 to 52.9~GB, Muon optimizer state from 14.60 to 6.77~GB, and the Muon optimizer step from 1288 to 304~ms.}
}

\begin{document}
\begin{mainpart}

\section{Introduction}
LLM pretraining is a computationally expensive stage of language-model development, where optimization quality, training time, and memory requirements jointly determine the models that can be trained within a given budget.  The optimizer affects all three through its update rule, computational cost, and state storage. Therefore, improving optimization must be accompanied by efficient use of memory and hardware.

AdamW~\citep{loshchilov2019decoupledweightdecayregularization} has long been the standard optimizer for large language models. Its coordinatewise updates are inexpensive and straightforward to shard across devices, but its first- and second-moment buffers account for a substantial share of training memory. GaLore~\citep{zhao2024galorememoryefficientllmtraining}, APOLLO~\citep{zhu2025apollosgdlikememoryadamwlevel}, and FRUGAL~\citep{zmushko2025frugalmemoryefficientoptimizationreducing} reduce this storage through compressed statistics or alternative update rules. Such savings can come at the cost of expensive projection operations or slower convergence, while leaving dense weights and gradients unchanged.

Muon~\citep{jordan2024muon} improves convergence by making an update geometry suited to matrix parameters, while also reducing optimizer state. It replaces coordinatewise adaptation with updates that approximately follow the polar factor of the momentum matrix, corresponding to spectral-norm steepest descent~\citep{bernstein2024oldoptimizernewnorm}. The main cost of Muon is orthogonalization, which involves expensive matrix multiplications and, in sharded implementations, additional communication. Recent methods accelerate polar computation or modify the update and its distributed implementation~\citep{amsel2026polarexpressoptimalmatrix, khaled2025muonbpfastermuonblockperiodic, ahn2025diondistributed, chen2026dmuonefficientdistributedmuon}. While these advances make matrix-aware updates cheaper, they retain dense weights and gradients.

The parameterization itself offers another route to efficiency. Reducing the number of trainable parameters reduces weights, gradients, and optimizer buffers together. For Muon, a suitable structure can also replace large orthogonalization problems with smaller ones. The important question is whether this computationally cheaper representation retains the optimization benefit of matrix-aware updates.

\begin{figure}[!htbp]
\centering
\includegraphics[width=0.9\textwidth]{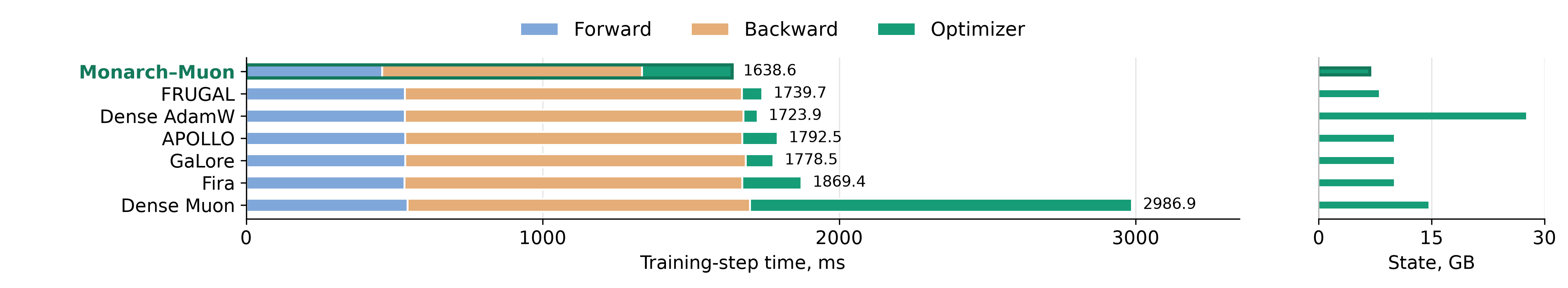}
\vspace{-2mm}
\caption{Training-step breakdown and optimizer-state memory for the 6.89B dense reference at microbatch 4 and four accumulation steps, from the same runs as Table~\ref{tab:memory-time}. Bars show median phase timings; labels give median full-step wall time. Periodic projection refreshes are excluded. Monarch--Muon uses four blocks and reduces both optimizer-state memory and optimizer-step overhead.}
\label{fig:step-time-breakdown}
\end{figure}

We introduce Monarch--Muon, which parameterizes selected weights with Monarch factors~\citep{dao2022monarchexpressivestructuredmatrices} and applies Muon directly to their trainable blocks. This representation uses fewer parameters while retaining the ability to express full-rank transformations, reducing storage for weights, gradients, and optimizer state. The smaller blocks also make Muon's orthogonalization cheaper. In dense and mixture-of-experts pretraining, Muon improves validation loss over AdamW at fixed Monarch structure. Systems measurements show that the computational savings grow with matrix size, making the approach particularly effective at larger scales.

\section{Related Work}

\subsection{Matrix-aware optimization}
Muon~\citep{jordan2024muon} demonstrates that changing the update geometry can improve language-model training without maintaining coordinatewise second moments. Its polar update corresponds to spectral-norm steepest descent~\citep{bernstein2024oldoptimizernewnorm}. However, computing spectral updates remains costly for large matrices. The Polar Express~\citep{amsel2026polarexpressoptimalmatrix} and Gram Newton--Schulz~\citep{GramNewtonSchulz} reduce the cost of polar approximation. NorMuon~\citep{li2025normuonmakingmuonefficient} further improves optimization by combining orthogonalization with neuron-wise scaling.

Sharded training introduces another obstacle: dense orthogonalization can require communication beyond the local parameter shard. Dion~\citep{ahn2025diondistributed} addresses this through low-rank orthonormalized updates with error feedback. MuonBP~\citep{khaled2025muonbpfastermuonblockperiodic} instead orthogonalizes shards independently and periodically performs full orthogonalization. DMuon~\citep{chen2026dmuonefficientdistributedmuon} improves distributed execution, while Dion3~\citep{amsel2026dion3fullstackorthogonalupdates} combines efficient kernels, megabatching, and partial orthogonalization. Their trainable weights remain dense, leaving an opportunity to reduce storage and computation through the parameterization itself.

\subsection{Memory-efficient optimizer state}

A complementary line reduces the state associated with dense parameters. GaLore~\citep{zhao2024galorememoryefficientllmtraining} stores optimizer statistics in low-rank gradient subspaces, substantially reducing memory. However, projection maintenance adds computation, and restricting updates to these subspaces can discard useful gradient directions. APOLLO~\citep{zhu2025apollosgdlikememoryadamwlevel} uses random projections to approximate adaptive scaling without costly SVD computations. Fira~\citep{chen2025firaachievefullranktraining} extends low-rank adaptive scaling to full-rank gradients, while FRUGAL~\citep{zmushko2025frugalmemoryefficientoptimizationreducing} combines stateful updates in selected directions with state-free updates in the remainder. These advances improve the memory--quality tradeoff of optimizer-state compression, but retain dense weights and gradients. 

\subsection{Structured parameterizations}

Monarch~\citep{dao2022monarchexpressivestructuredmatrices} provides such a representation through products of block-diagonal factors and permutations. It supports full-rank transformations with fewer parameters and demonstrates efficient training of structured models. Compute Better Spent~\citep{qiu2024compute} extends the study of structured layers, including Monarch and block tensor-train matrices, and shows that their performance depends strongly on suitable initialization and learning-rate scaling. Thus, choosing an efficient representation also raises a question: how should its factors be updated?

Factorized adaptation methods address a related question for pretrained models. LoRA~\citep{hu2022lora} reduces trainable parameters through low-rank updates, while LoRA-Muon~\citep{cesista2026lora} develops spectral steepest descent for that low-rank setting. Monarch--Muon targets pretraining with full-rank-capable structured weights. It applies matrix-aware updates directly to the Monarch blocks, using the same structure to reduce training storage and localize orthogonalization. Figure~\ref{fig:step-time-breakdown} illustrates the resulting savings in training-step time and optimizer-state memory at the largest benchmark scale.

\paragraph{Contributions.}
We make the following contributions:
\vspace{-2mm}
\begin{itemize}[nosep]
\item \textbf{Monarch--Muon optimizer.} We introduce Monarch--Muon, a matrix-aware optimizer built on the Monarch parameterization, which reduces weight, gradient, and optimizer-state storage and replaces dense orthogonalization with smaller independent computations.

\item \textbf{Theoretical analysis.} We derive a blockwise linear-minimization decomposition and prove convergence of exact-polar Monarch--Muon under generalized smoothness.

\item \textbf{Experimental validation.} Monarch--Muon improves validation loss over Monarch--AdamW in dense and MoE pretraining, matching dense AdamW in MoE with 15\% fewer active and 25.4\% fewer total parameters. Large-scale benchmarks show 54\% lower optimizer-state memory and 45\% shorter training steps than dense Muon.

\item \textbf{Initialization analysis.} To our knowledge, we provide the first systematic study of dense-to-Monarch initialization for SFT, showing that activation-aware fitting and orthogonal rotations improve post-SFT quality over closed-form weight projection under both AdamW and Muon.
\end{itemize}

\section{Monarch--Muon}
\label{sec:monarch-muon}

Monarch--Muon applies matrix-aware updates to the trainable factors of structured weights. Its block structure reduces training storage and replaces dense orthogonalization with independent polar computations on smaller matrices.

\subsection{Algorithm}
\label{sec:monarch-muon-algorithm}

For a selected weight $W\in\mathbb R^{d_{\mathrm{out}}\times d_{\mathrm{in}}}$, we use the Monarch parameterization~\citep{dao2022monarchexpressivestructuredmatrices}
\begin{equation}
\label{eq:monarch-parameterization}
\begin{aligned}
\widehat W(\theta)&=P_1L(\theta)P_2R(\theta),\\
L(\theta)&=\operatorname{bdiag}(X_1,\ldots,X_m),\\
R(\theta)&=\operatorname{bdiag}(X_{m+1},\ldots,X_{2m}).
\end{aligned}
\end{equation}
Here $P_1,P_2$ are fixed, dimension-compatible perfect-shuffle permutations, and $\theta=(X_1,\ldots,X_K)$ contains the $K=2m$ trainable blocks of both factors. The forward pass applies $R$, $P_2$, $L$, and $P_1$ successively, without assembling a dense weight. For clarity, Algorithm~\ref{alg:monarch-muon} describes one converted weight. The same update is applied to every selected projection.

Each block has its own momentum buffer, initialized to zero (Line~\ref{line:mm-init}). A minibatch loss is evaluated through the structured weight, and backpropagation computes the gradients of both factors at the same current parameters (Lines~\ref{line:mm-sample}--\ref{line:mm-backward}). Thus, the factor updates use a common forward and backward pass.

The gradient of each block is accumulated into its momentum (Line~\ref{line:mm-momentum}). We then approximate the polar factor of this momentum using $s$ Newton--Schulz iterations (Line~\ref{line:mm-polar}) and update the block (Line~\ref{line:mm-update}). For a compact singular value decomposition $M_j=U\Sigma V^\top$ containing the positive singular values, the exact direction is $\operatorname{Polar}(M_j)=UV^\top$. The operator $\operatorname{NS}_s$ computes a polynomial approximation to this direction after normalizing the momentum matrix.

Once the gradients are available, the momentum and polar computations are independent across blocks. Blocks of equal shape can therefore be processed in batched matrix multiplications. Parameters outside the Monarch factors use AdamW, and the training schedules and practical Muon variants are specified in Appendix~\ref{app:setup-pretrain}.

\begin{algorithm}{\linewidth}{Blockwise Monarch--Muon}
\label{alg:monarch-muon}
\begin{list}{}{\setlength{\leftmargin}{1.8em}\setlength{\labelwidth}{1.4em}\setlength{\labelsep}{0.4em}\setlength{\itemsep}{1pt}\setlength{\parsep}{0pt}}
\setcounter{algorithmline}{0}
\item[] \textbf{Input:} Initial blocks $\theta^0=(X_1^0,\ldots,X_K^0)$, $K=2m$; fixed permutations $P_1,P_2$; steps $T$
\item[] Settings: step sizes $\eta_k$, momentum coefficients $\beta_k$, Newton--Schulz iterations $s$
\mmline $M_j^{-1}\gets0$ for $j=1,\ldots,K$ \label{line:mm-init}
\mmline \textbf{for} $k=0,\ldots,T-1$ \textbf{do}
    \mmline \hspace*{1em}Sample a minibatch $\xi^k$ \label{line:mm-sample}
    \mmline \hspace*{1em}$\ell_k\gets f(\widehat W(\theta^k);\xi^k)$ \label{line:mm-forward}
    \mmline \hspace*{1em}$(G_1^k,\ldots,G_K^k)\gets\nabla_\theta\ell_k$ \label{line:mm-backward}
    \mmline \hspace*{1em}\textbf{for all} $j=1,\ldots,K$ in parallel \textbf{do}
        \mmline \hspace*{2em}$M_j^k\gets\beta_kM_j^{k-1}+(1-\beta_k)G_j^k$ \label{line:mm-momentum}
        \mmline \hspace*{2em}$U_j^k\gets\operatorname{NS}_s(M_j^k)$ \label{line:mm-polar}
        \mmline \hspace*{2em}$X_j^{k+1}\gets X_j^k-\eta_kU_j^k$ \label{line:mm-update}
    \mmline \hspace*{1em}\textbf{end for}
\mmline \textbf{end for}
\mmline \textbf{return} trained factors $L(\theta^T),R(\theta^T)$
\end{list}
\end{algorithm}

\subsection{Theoretical Analysis}
\label{sec:theory}

\paragraph{Representation and computational cost.}
Consider the square case $\widehat W=PLP^\top R$, where each factor has $m$ blocks of size $d/m\times d/m$. If the blocks are nonsingular,
$$
\det(\widehat W)=\prod_{i=1}^{m}\det(L_i)\det(R_i)\neq0.
$$
Hence the parameterization can represent full-rank weights with $2d^2/m$ parameters. For $m>2$, this is fewer than the $d^2$ parameters of a dense weight. Compression restricts the matrix structure without imposing a low-rank constraint.

The same parameter reduction applies to gradients and block momentum. For rectangular weights, writing $h=\min(d_{\mathrm{in}},d_{\mathrm{out}})$ and $D=\max(d_{\mathrm{in}},d_{\mathrm{out}})$, the factor storage is $(h^2+hD)/m$ parameters. Orthogonalization benefits from the smaller matrix dimensions. A Newton--Schulz iteration costs $O(d^3)$ for a dense square matrix, while the $2m$ factor blocks require
$$
2m\,O\!\left((d/m)^3\right)=O\!\left(2d^3/m^2\right).
$$
Thus, increasing the block count reduces both storage and the arithmetic cost of the polar computation.

\paragraph{Blockwise update geometry.}
For a tuple $H=(H_1,\ldots,H_K)$, define the block norm and its dual
\begin{equation}
\label{eq:lmo-block-geometry}
\begin{aligned}
\|H\|_{\mathrm b}&=\max_j\|H_j\|_{\op},\qquad
\|G\|_{\mathrm b,*}=\sum_j\|G_j\|_*.
\end{aligned}
\end{equation}
Here $\|\cdot\|_{\op}$ and $\|\cdot\|_*$ are the spectral and nuclear norms, respectively, and $\langle G,H\rangle=\sum_j\langle G_j,H_j\rangle_F$ is the sum of Frobenius inner products. A radius-$\eta$ linear minimization oracle in this factor space separates as
\begin{equation}
\label{eq:blockwise-muon-decomposition}
\begin{aligned}
\min_{\|\Delta\|_{\mathrm b}\le\eta}\langle M,\Delta\rangle
&=\sum_{j=1}^{K}\min_{\|\Delta_j\|_{\op}\le\eta}\langle M_j,\Delta_j\rangle_F=-\eta\sum_{j=1}^{K}\|M_j\|_*.
\end{aligned}
\end{equation}
An exact minimizer is $\Delta_j=-\eta\operatorname{Polar}(M_j)$. This gives the blockwise Muon update: each direction solves a spectral-norm subproblem, and all blocks share the same step radius. The decomposition applies jointly to the blocks of $L$ and $R$, even though their gradients depend on both factors. Lemma~\ref{lem:blockwise-muon-step} in Appendix~\ref{app:theory-proofs} proves the decomposition.

\paragraph{Convergence.}
\label{sec:lmo-convergence}
We establish convergence of the exact-polar update for the coupled factor objective $F(\theta)=\mathbb E_\xi f(\widehat W(\theta);\xi)$.

\begin{theorem}[Convergence of exact-polar Monarch--Muon]
\label{thm:monarch-muon-main}
Let $F$ be continuously differentiable and bounded below by $F_{\inf}$, and let $\theta^0$ be deterministic. Suppose that $F$ satisfies generalized smoothness~\eqref{eq:lmo-smoothness} with constants $\mathcal L_0,\mathcal L_1\ge0$, and the stochastic gradients satisfy conditional unbiasedness and bounded conditional Frobenius variance~\eqref{eq:lmo-noise}. Run Algorithm~\ref{alg:monarch-muon} with exact polar factors and schedules
$$
\beta_k=1-(k+1)^{-1/2},\qquad \eta_k=\eta_0(k+1)^{-3/4},
$$
where $\eta_0=(12\mathcal L_1)^{-1}$ if $\mathcal L_1>0$, and $\eta_0$ is any fixed positive value if $\mathcal L_1=0$. Then, for every $T\ge1$,
\begin{equation}
\label{eq:lmo-convergence-bound}
\min_{0\le k<T}\mathbb E\!\left[\sum_{j=1}^{K}\|\nabla_{X_j}F(\theta^k)\|_*\right]=O\!\left(\frac{1+\log T}{T^{1/4}}\right).
\end{equation}
The implicit constant is independent of $T$.
\end{theorem}

Theorem~\ref{thm:monarch-muon-main} specializes the Gluon analysis~\citep{riabinin2025gluon} to the coupled factor objective and yields $\widetilde O(\varepsilon^{-4})$ iterations to reach expected factor-gradient norm at most $\varepsilon$. Appendix~\ref{app:monarch-convergence-statement} states the assumptions in full, and Appendix~\ref{app:convergence-proofs} proves the result.

\section{Experiments}
\label{sec:experiments}
\setlength{\intextsep}{6pt plus 2pt minus 1pt}
\setlength{\textfloatsep}{10pt plus 2pt minus 2pt}
\setlength{\floatsep}{6pt plus 2pt minus 1pt}

\begin{wraptable}{r}{0.56\textwidth}
\centering
\vspace{-4mm}
\caption{Pretraining on FineWeb-Edu under a common training recipe within each architecture.}
\vspace{-2mm}
\label{tab:quality}
\vspace{2mm}
\scriptsize
\setlength{\tabcolsep}{2.5pt}
\renewcommand{\arraystretch}{0.98}
\maxsizebox{\linewidth}{!}{%
\begin{tabular}{lccc}
\toprule
\textbf{Optimizer} & \textbf{Val. loss} & \textbf{Active params} & \textbf{Total params} \\
\midrule
\multicolumn{4}{l}{\textbf{Dense 257M}} \\
AdamW & 3.02 & 257M & 257M \\
Muon & 2.99 & 257M & 257M \\
GaLore & 3.07 & 257M & 257M \\
Monarch--AdamW, 2 blocks & 3.09 & 224M & 224M \\
Monarch--Muon, 2 blocks & \textbf{3.06} & 224M & 224M \\
Monarch--AdamW, 4 blocks & 3.16 & 164M & 164M \\
Monarch--Muon, 4 blocks & \textbf{3.14} & 164M & 164M \\
\midrule
\multicolumn{4}{l}{\textbf{MoE}} \\
AdamW & 2.68 & 280M & 1.029B \\
Muon & 2.65 & 280M & 1.029B \\
Monarch--AdamW, 2 blocks & 2.77 & 237M & 767M \\
Monarch--Muon, 2 blocks & \textbf{2.68} & 237M & 767M \\
Monarch--AdamW, 4 blocks & 2.79 & 171M & 436M \\
Monarch--Muon, 4 blocks & \textbf{2.75} & 171M & 436M \\
\bottomrule
\end{tabular}}
\vspace{-4mm}
\end{wraptable}
We evaluate Monarch--Muon through four comparisons: optimization quality at fixed Monarch structure, training costs as model width increases, initialization for dense-to-Monarch transfer, and compatibility with FP8. Quality runs train a 257M dense reference and a 1B MoE reference on FineWeb-Edu~\citep{penedo2024fineweb}. Systems benchmarks cover dense reference sizes up to 6.89B on an H100. These benchmarks measure the cost of the training operations at larger widths, while the complete pretraining runs test learning quality. Appendix~\ref{app:setup} gives the protocols.

\subsection{Pretraining quality at fixed structure}
\label{sec:exp-pretrain}

\begin{wrapfigure}{r}{0.56\textwidth}
\centering
\includegraphics[width=\linewidth]{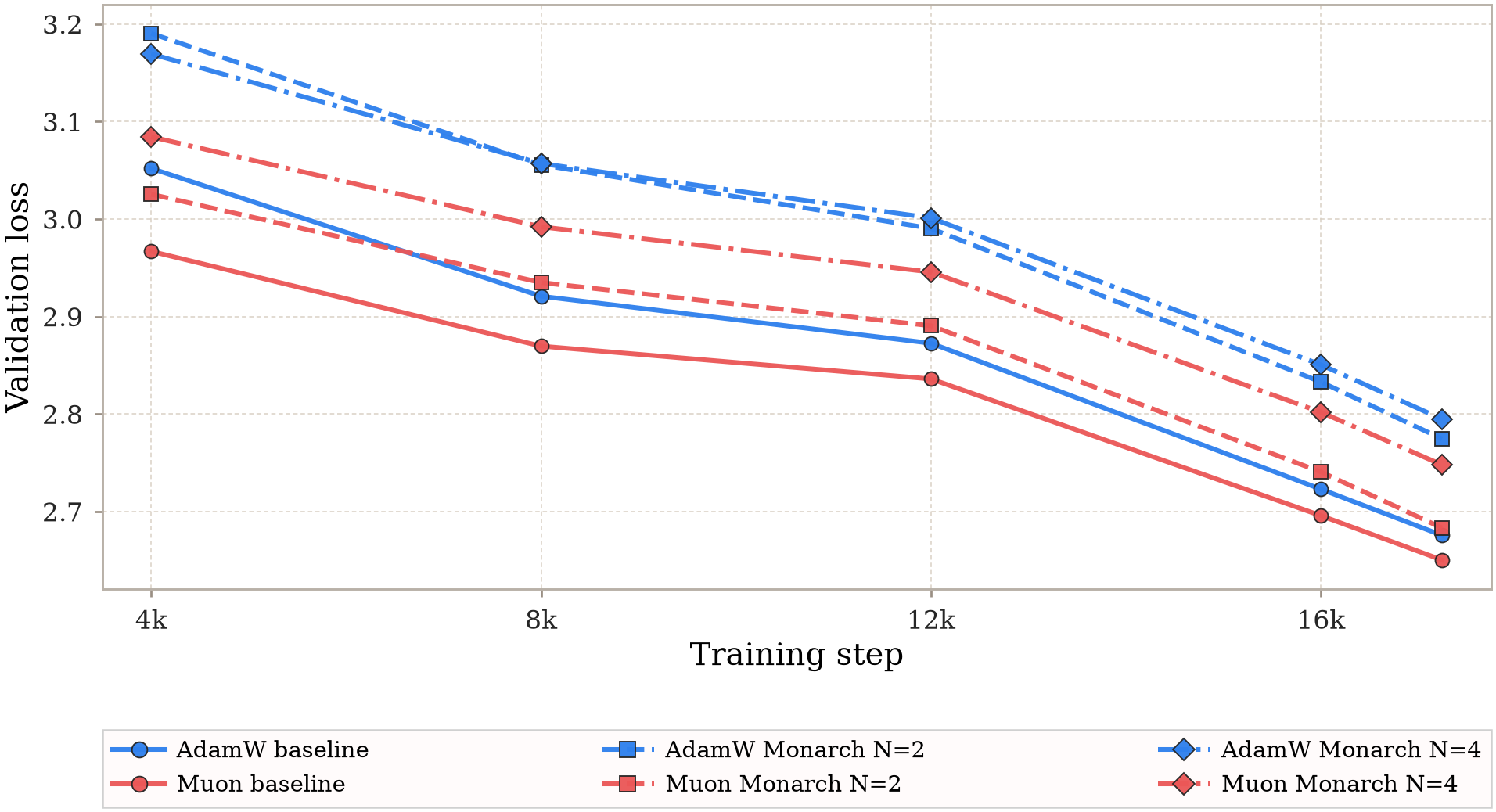}
\caption{MoE pretraining at the same token budget. Two-block Monarch--Muon reaches the reported final validation loss of dense AdamW with fewer active and total parameters.}
\vspace{-2mm}
\label{fig:moe-quality}
\end{wrapfigure}

Holding the Monarch factorization fixed isolates the effect of the optimizer. We compare AdamW and Muon with two and four blocks, using the same token budget and training schedule within each architecture. Dense AdamW, dense Muon, and GaLore provide reference losses for the dense quality model. The MoE comparison retains the same routing architecture and changes only the selected projections and their optimizer.

Monarch--Muon achieves lower final validation loss than Monarch--AdamW in all four structured comparisons (Table~\ref{tab:quality}). The two-block MoE reaches the reported loss of dense AdamW with 15\% fewer active parameters and 25.4\% fewer total parameters (767M versus 1.029B). Figure~\ref{fig:moe-quality} places this endpoint in the training trajectory. 
Thus, matrix-aware updates remain useful when the optimizer acts on structured factors rather than dense weights. Increasing the block count further reduces storage, with a higher final loss under both optimizers. Dense Muon gives the lowest loss in each architecture, defining the quality reference for this representation tradeoff.

The learned spectra provide a complementary view of the optimizer effect. At the same block count, Monarch--Muon retains a heavier singular-value tail than Monarch--AdamW in the gate and up projections (Figure~\ref{fig:monarch-gate-up-spectra}). Similar patterns appear in additional layers and components (Appendix~\ref{app:additional-spectra}). Thus, the change in update geometry reaches the represented matrices, even though orthogonalization is local to their factors.

\begin{figure}[t]
\centering
\includegraphics[width=0.8\columnwidth]{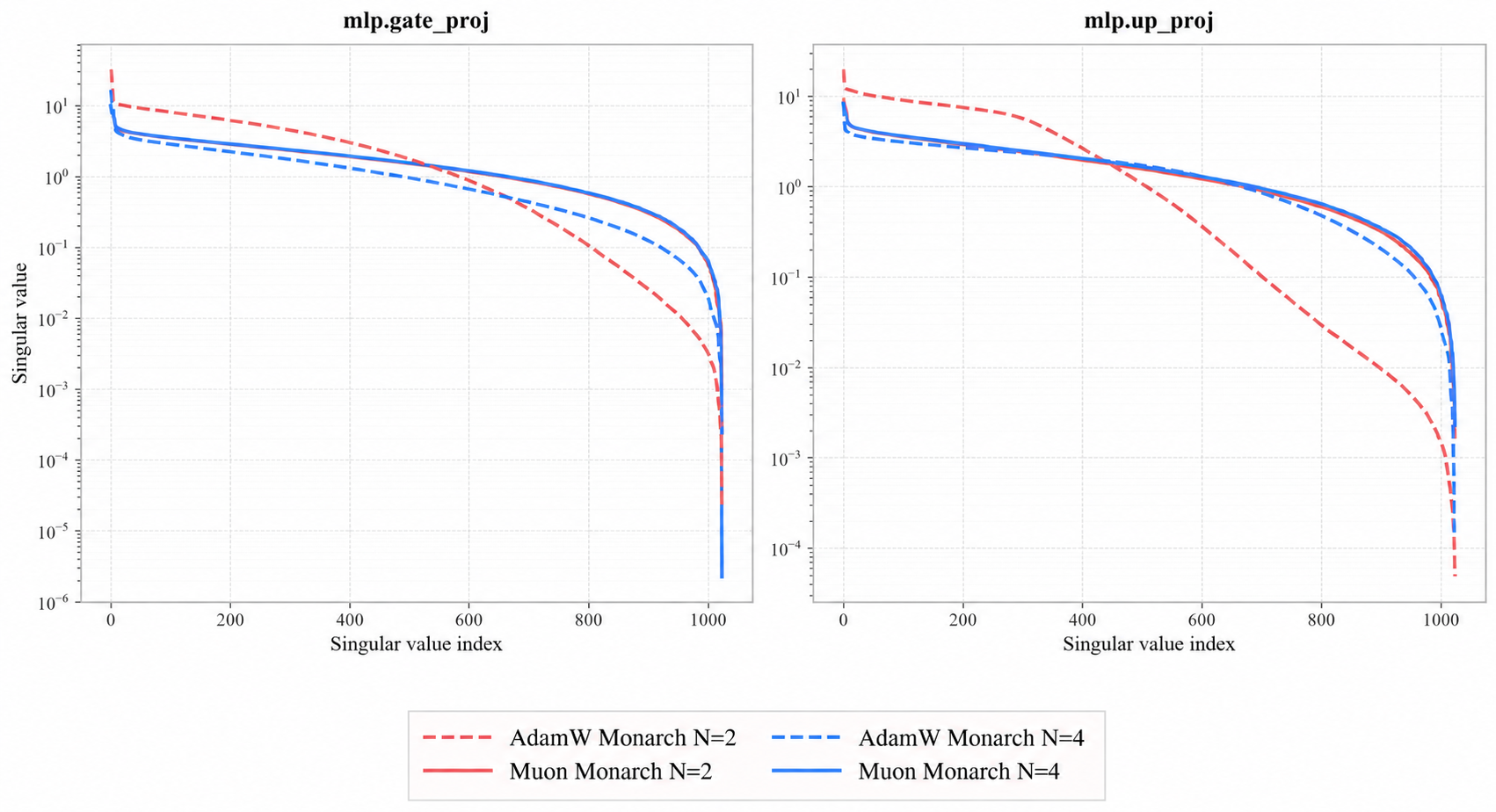}
\caption{Learned gate and up projections in layer 0 of the dense quality model. At the same block count, Muon retains larger singular values in the tail than AdamW. Additional layers and components are in Appendix~\ref{app:additional-spectra}.}
\label{fig:monarch-gate-up-spectra}
\end{figure}

\subsection{Training efficiency improves with width}
\label{sec:exp-systems}

Structure reduces more than the optimizer state: it also compresses the trainable weights and gradients. For Muon, the smaller blocks additionally reduce the cost of orthogonalization. Section~\ref{sec:theory} derives the distinct scaling laws: square-factor storage decreases as $2/m$ of the dense matrix, while Newton--Schulz arithmetic decreases as $2/m^2$, where $m$ is the block count. We test how these savings translate to allocated memory and wall time across five dense reference sizes, using four-block Monarch models and a common microbatch of four. Table~\ref{tab:memory-time} includes both optimizer and full-step measurements. Timings for projection-based methods exclude their periodic projection refreshes.

The benefit grows with width because cheaper orthogonalization is increasingly accompanied by cheaper forward and backward passes. At 6.89B, Monarch--Muon reduces the full-step time by 45\% and peak memory by 23\% relative to dense Muon. Its optimizer step is $4.2\times$ faster and its state occupies 54\% less memory. Figure~\ref{fig:step-time-breakdown} shows that both the model computations and the optimizer contribute to the full-step reduction. At this scale, Monarch--Muon also has the shortest full step and lowest peak memory among the measured state-efficient methods, even though their optimizer steps alone are cheaper. Compressing the representation therefore changes costs that optimizer-state compression leaves dense.

The smallest model marks the hardware crossover: its reduced optimizer cost does not offset the overhead of executing small factor multiplications, and peak memory is slightly higher than for dense Muon. From the 834M reference onward, Monarch--Muon improves both full-step time and peak memory relative to dense Muon. Full-step measurements capture this crossover, which optimizer-step timings alone would miss.

\begin{table}[t]
\centering
\caption{Training costs on one H100 at microbatch 4, indexed by the dense reference model size. Monarch--Muon uses 4 blocks. Top: peak allocated memory and optimizer buffers, including projection matrices (GB). Bottom: steady-state training-step time, with optimizer-step time in parentheses (ms).}
\label{tab:memory-time}
\footnotesize
\setlength{\tabcolsep}{3pt}
\renewcommand{\arraystretch}{0.94}
\maxsizebox{\linewidth}{!}{%
\begin{tabular}{l*{10}{r}}
\toprule
& \multicolumn{2}{c}{\textbf{257M}} & \multicolumn{2}{c}{\textbf{834M}} & \multicolumn{2}{c}{\textbf{1.4B}} & \multicolumn{2}{c}{\textbf{3.5B}} & \multicolumn{2}{c}{\textbf{6.9B}} \\
\cmidrule(lr){2-3}\cmidrule(lr){4-5}\cmidrule(lr){6-7}\cmidrule(lr){8-9}\cmidrule(l){10-11}
\textbf{Method} & \textbf{Peak} & \textbf{State} & \textbf{Peak} & \textbf{State} & \textbf{Peak} & \textbf{State} & \textbf{Peak} & \textbf{State} & \textbf{Peak} & \textbf{State} \\
\midrule
Dense Muon & 6.26 & 0.72 & 14.35 & 1.98 & 20.51 & 3.29 & 39.71 & 7.58 & 68.87 & 14.60 \\
Monarch--Muon & 6.38 & 0.53 & 14.07 & 1.16 & 18.94 & 1.79 & 33.15 & 3.75 & 52.94 & 6.77 \\
GaLore & 6.17 & 0.61 & 13.95 & 1.50 & 19.72 & 2.41 & 37.47 & 5.33 & 64.27 & 10.00 \\
FRUGAL & 6.11 & 0.57 & 13.73 & 1.36 & 19.40 & 2.17 & 36.35 & 4.22 & 62.28 & 8.01 \\
APOLLO & 6.16 & 0.61 & 13.93 & 1.50 & 19.72 & 2.41 & 37.47 & 5.33 & 64.27 & 10.00 \\
Fira & 6.16 & 0.61 & 13.96 & 1.50 & 19.72 & 2.41 & 37.47 & 5.33 & 64.27 & 10.00 \\
Dense AdamW & 6.59 & 1.03 & 15.71 & 3.34 & 22.98 & 5.76 & 46.05 & 13.92 & 81.82 & 27.55 \\
\bottomrule
\end{tabular}}

\vspace{6pt}
\begin{tabular}{l*{5}{r}}
\toprule
\textbf{Method} & \textbf{257M} & \textbf{834M} & \textbf{1.4B} & \textbf{3.5B} & \textbf{6.9B} \\
\midrule
Dense Muon & 146.5 (26.2) & 428.2 (88.9) & 650.4 (156.7) & 1473.0 (505.2) & 2986.9 (1288.3) \\
Monarch--Muon & 165.8 (8.0) & 399.8 (29.2) & 535.7 (50.5) & 997.9 (140.5) & 1638.6 (304.4) \\
GaLore & 126.4 (6.1) & 354.6 (14.6) & 511.8 (20.4) & 1005.7 (45.2) & 1778.5 (93.9) \\
FRUGAL & 124.9 (4.9) & 349.9 (11.6) & 507.5 (16.8) & 989.8 (36.0) & 1739.7 (69.2) \\
APOLLO & 132.6 (12.5) & 365.3 (27.7) & 525.7 (32.5) & 1016.0 (63.1) & 1792.5 (119.3) \\
Fira & 136.6 (16.2) & 372.1 (34.6) & 538.8 (47.3) & 1055.8 (99.0) & 1869.4 (199.8) \\
Dense AdamW & 120.9 (1.9) & 345.3 (5.7) & 502.9 (9.9) & 986.2 (24.1) & 1723.9 (47.5) \\
\bottomrule
\end{tabular}

\end{table}

\newpage
\subsection{Initialization for dense-to-Monarch transfer}
\label{sec:exp-finetune}

Pretraining learns within the structured representation from the outset. Converting a dense checkpoint instead requires choosing which transformations to replace and how to initialize their factors. We study these choices on the Muon- and AdamW-trained Moonlight-16B-A3B checkpoints~\citep{liu2025muonscalable} and then test initialization followed by SFT on Qwen2.5-7B~\citep{qwen2025qwen25technicalreport}.

\paragraph{Joint sensitivity guides conversion.}
\label{sec:exp-layerwise}

Closed-form Monarch projection~\citep{dao2022monarchexpressivestructuredmatrices} minimizes weight-reconstruction error for a fixed factorization. Functional sensitivity identifies which layers can be converted with little change to model predictions. Figure~\ref{fig:moonlight-singles} measures KL divergence to each original checkpoint after replacing one FFN with its four-block projection. The first dense FFN and the final MoE FFN are substantially more sensitive than the interior MoE layers. Both checkpoints are evaluated at step 42{,}000 on 32 held-out sequences of 2048 tokens, without retraining (Appendix~\ref{app:setup-moonlight}). Weight-projection errors do not capture this variation, and interactions between converted layers require candidate sets to be validated jointly. Appendix~\ref{app:moonlight-conversion-analysis} provides the weight-space diagnostics and joint-conversion results.

\begin{figure}[!htbp]
\centering
\includegraphics[width=0.6\linewidth]{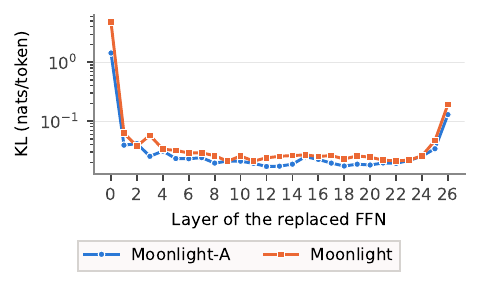}
\caption{Moonlight checkpoints: KL to the original model after replacing one FFN by its 4-block projection (nats/token, log scale). Layer 0 is dense. Layers 1--26 are MoE.}
\label{fig:moonlight-singles}
\end{figure}
\FloatBarrier

\paragraph{Fit layer outputs before fine-tuning.}
Qwen SFT confirms that the choice of components matters beyond the initial projection. We fine-tune converted models for one epoch on 10{,}000 examples from the Tulu~3 SFT mixture~\citep{lambert2025tulu3}. We measure response negative log-likelihood (NLL) on held-out test data and select learning rates on a separate validation split. Table~\ref{tab:sft-selective} separates the effects of layer range and projection type. Converting gate and up projections gives lower post-SFT NLL than converting the full MLP at similar parameter counts, under both optimizers. This ordering persists when gate+up conversion spans more layers and uses fewer parameters than the full-MLP comparison. The choice of transformation therefore matters beyond the total number of factors. In particular, preserving the down projection improves the tested quality--storage tradeoff. Full conversion retains a larger NLL gap after SFT, motivating selective conversion followed by calibration.

\begin{table}[!htbp]
\centering
\caption{Qwen2.5-7B: selective 2-block conversion followed by SFT. Pre-SFT NLL uses validation data. Post-SFT test NLL uses validation-selected learning rates.}
\label{tab:sft-selective}
\footnotesize
\setlength{\tabcolsep}{2.2pt}
\maxsizebox{\linewidth}{!}{%
\begin{tabular}{lrrrrr}
\toprule
& & & \textbf{NLL} & \multicolumn{2}{c}{\textbf{Test NLL after SFT}} \\
\cmidrule(l){5-6}
\textbf{Converted} & \textbf{\#mat.} & \textbf{Params} & \textbf{before} & \textbf{AdamW} & \textbf{Muon} \\
\midrule
None (dense) & 0 & 7.62B & 0.737 & 0.631 & 0.643 \\
All projections & 196 & 5.26B & 9.951 & 1.164 & 1.125 \\
\midrule
MLP, layers 14--26 & 39 & 6.54B & 2.181 & 0.804 & 0.811 \\
MLP, layers 8--26 & 57 & 6.05B & 2.784 & 0.858 & 0.866 \\
Gate+up, layers 14--26 & 26 & 6.90B & 1.362 & 0.731 & 0.741 \\
Gate+up, layers 8--26 & 38 & 6.57B & 1.729 & 0.762 & 0.773 \\
Gate+up, layers 2--26 & 50 & 6.24B & 3.19 & 0.791 & 0.804 \\
\bottomrule
\end{tabular}}
\end{table}

Weight-space fitting treats every input direction equally. Activation-aware fitting instead weights the error by the calibration activations. Appendix~\ref{app:qwen-weight-diagnostics} reports the weight-space diagnostic motivating this choice. We alternate between the two Monarch factors on 128 calibration examples and select the fit on 32 held-out examples, with the full procedure in Appendix~\ref{app:setup-finetune}.

\FloatBarrier
 
Table~\ref{tab:conversion} isolates initialization with the same Monarch-factor shapes. Activation-aware fitting~\citep{nakamura2026improvedmonarch} reduces the layer-output error by 40\% and closes 76\% of the pre-SFT NLL gap to the dense reference. Its advantage remains after SFT with both AdamW and Muon. Orthogonal rotations~\citep{grishina2025procrustesgpt} improve recovery further, at the additional computation described in Appendix~\ref{app:setup-finetune}. The tested permutation search~\citep{mohamed2025learningpermutations} leaves the initialization unchanged. The consistent ordering before and after SFT identifies preservation of layer outputs as a useful criterion for initializing structured factors.

\begin{table}[H]
\centering
\caption{Qwen MLPs 14--26 with fixed 2-block factor shapes. Output error: layer mean of summed relative squared MLP errors. Pre-SFT: validation. Post-SFT: test NLL with validation-selected rates. Rotations add dense products.}
\label{tab:conversion}
\footnotesize
\setlength{\tabcolsep}{2.2pt}
\maxsizebox{\linewidth}{!}{%
\begin{tabular}{lrrrrr}
\toprule
& \textbf{Output} & \multicolumn{2}{c}{\textbf{Before SFT}} & \multicolumn{2}{c}{\textbf{Test NLL after SFT}} \\
\cmidrule(lr){3-4}\cmidrule(l){5-6}
\textbf{Initialization} & \textbf{err.} & \textbf{NLL} & \textbf{KL} & \textbf{AdamW} & \textbf{Muon} \\
\midrule
Dense model & -- & 0.737 & 0 & 0.630 & 0.638 \\
A: closed-form projection & 0.595 & 2.181 & 1.52 & 0.802 & 0.813 \\
B: output fit & 0.357 & 1.083 & 0.41 & 0.765 & 0.775 \\
C: permutation + B & 0.357 & 1.083 & 0.41 & 0.765 & 0.775 \\
D: rotations + B & 0.315 & 1.024 & 0.34 & 0.748 & 0.758 \\
\bottomrule
\end{tabular}}
\end{table}

Initialization changes post-SFT loss more than the choice between the two factor optimizers in this ablation. AdamW gives slightly lower loss on the selective conversions of this AdamW-pretrained checkpoint, consistent with the observations on retaining the pretraining optimizer in \citet{liu2025muonscalable}. The transfer study therefore complements the pretraining result: matrix-aware updates improve learning from scratch, while selective conversion and activation-aware initialization improve recovery of pretrained functions.

\subsection{FP8 complements structured training}
\label{sec:exp-lowbit}

Low precision reduces the bytes and arithmetic cost of the representation that Monarch has already compressed. Following the FP8 approach of COAT~\citep{xi2025coat}, we evaluate two separate extensions of the two-block dense quality model: FP8 factor GEMMs with compressed activations, and FP8 optimizer buffers with BF16 GEMMs. Both use the same 39{,}250-step training recipe, with final loss evaluated in BF16 (Appendix~\ref{app:setup-lowbit}).

The FP8 computation and activation recipe preserves the optimizer ordering while reducing peak memory by about 31\% and training-step time by 14\% for Monarch--Muon. Final loss remains within 0.23\% of BF16 for both optimizers (Table~\ref{tab:lowbit-full}). Quantizing optimizer buffers separately gives nearly unchanged loss and reduces their storage by $3.7\times$. These results show two complementary uses of low precision: reducing the larger activation footprint and further compressing the already smaller optimizer state.

\begin{table}[!htbp]
\centering
\caption{Two-block Monarch after 39{,}250 steps. FP8 GEMM compresses computation and activations. FP8 states retain BF16 GEMMs. Validation uses BF16. Peak memory: final step on one H100.}
\label{tab:lowbit-full}
\footnotesize
\setlength{\tabcolsep}{3pt}
\begin{tabular}{llrr}
\toprule
\textbf{Optimizer} & \textbf{Precision} & \textbf{Val. loss} & \textbf{Peak GB} \\
\midrule
Monarch--Muon & BF16 & 3.059 & 49.47 \\
 & FP8 GEMM & 3.066 & 34.28 \\
 & FP8 states & 3.060 & 48.52 \\
\midrule
Monarch--AdamW & BF16 & 3.088 & 49.97 \\
 & FP8 GEMM & 3.093 & 34.75 \\
 & FP8 states & 3.087 & 48.66 \\
\bottomrule
\end{tabular}
\end{table}

\begin{table}[!htbp]
\centering
\caption{Monarch RMSNorm + SwiGLU forward/backward time on an H100 (ms), with prequantized weights. Shapes: hidden / MLP / tokens. TE FP8 (Transformer Engine): best of three recipes. Batched FP8 speedups use TE BF16 at the same block count.}
\vspace{2mm}
\label{tab:lowbit-layer}
\footnotesize
\setlength{\tabcolsep}{2.2pt}
\maxsizebox{\linewidth}{!}{%
\begin{tabular}{llrrr}
\toprule
\textbf{Shape} & \textbf{Blocks} & \textbf{TE BF16} & \textbf{TE FP8} & \textbf{Ours FP8} \\
\midrule
1024 / 2048 / 26{,}624 & 2 & 4.44 & 5.04 (0.88$\times$) & 3.51 (1.26$\times$) \\
 & 4 & 4.67 & 5.64 (0.83$\times$) & 3.71 (1.26$\times$) \\
2048 / 6144 / 16{,}384 & 2 & 7.84 & 8.18 (0.96$\times$) & 6.01 (1.30$\times$) \\
 & 4 & 7.27 & 8.16 (0.89$\times$) & 5.70 (1.28$\times$) \\
2048 / 16{,}384 / 8192 & 2 & 8.35 & 8.49 (0.98$\times$) & 6.65 (1.25$\times$) \\
 & 4 & 7.42 & 8.25 (0.90$\times$) & 6.02 (1.23$\times$) \\
7168 / 16{,}384 / 8192 & 2 & 23.89 & 17.96 (1.33$\times$) & 16.33 (1.46$\times$) \\
 & 4 & 16.20 & 14.68 (1.10$\times$) & 11.77 (1.38$\times$) \\
\bottomrule
\end{tabular}}
\end{table}

Realizing the compute benefit requires execution suited to the factor shapes. Our implementation quantizes the shared input once and processes the blocks of each factor in one batched FP8 GEMM. Across the large MLP shapes in Table~\ref{tab:lowbit-layer}, this gives $1.23$--$1.46\times$ faster forward and backward passes than BF16 at the same block count. Standard Transformer Engine FP8 recipes do not consistently accelerate these shapes. At the largest shape, the four-block FP8 MLP is $2.22\times$ faster than the dense BF16 reference. The layer measurements identify batched execution as a practical complement to the savings of the factorization.

\FloatBarrier
\section{Conclusion}

Monarch--Muon combines a full-rank-capable structured representation with blockwise matrix-aware updates. This reduces weight, gradient, and optimizer-state storage and replaces dense orthogonalization with smaller independent computations. Our exact-polar analysis gives a convergence guarantee under generalized smoothness. In dense and MoE pretraining, Muon improves loss at fixed Monarch structure, with the two-block MoE matching the reported dense AdamW loss using fewer parameters. Systems measurements show increasing full-step savings at larger widths, complemented by FP8. Dense-checkpoint conversion benefits from jointly selecting components and initializing factors to preserve layer outputs.

\paragraph{Limitations and future work.}
Efficiency depends on matrix shape, block count, and kernel execution, motivating sensitivity-guided block allocation and batched low-precision kernels. Complete pretraining at controlled scales and systems measurements at larger widths provide complementary evidence. Longer distributed runs would test their joint quality and hardware tradeoff. Analysis of finite Newton--Schulz updates would connect the exact-polar guarantee to practical orthogonalization.



\end{mainpart}

\begin{appendixpart}

\section{Experimental Details}
\label{app:setup}

\paragraph{Assets and licenses.} We use FineWeb-Edu \citep{penedo2024fineweb} and the Tulu~3 SFT mixture \citep{lambert2025tulu3} (both ODC-By), Qwen2.5-7B \citep{qwen2025qwen25technicalreport} (Apache-2.0), the Moonlight models \citep{liu2025muonscalable} (MIT), MMLU \citep{hendrycks2021mmlu} and GSM8K \citep{cobbe2021gsm8k} (MIT) evaluated with the LM Evaluation Harness \citep{gao2024evalharness} (MIT), COAT \citep{xi2025coat} (Apache-2.0) and Megatron-Core \citep{shoeybi2019megatron} (BSD-3-Clause). Our code is released at the GitHub URL given in the abstract.

\subsection{Pretraining}
\label{app:setup-pretrain}
Table~\ref{tab:quality} compares final validation losses at a common training budget within each architecture. The AdamW and Muon Monarch variants share the same factor shapes and converted components. The comparisons use the common training recipes and learning rates specified below. Each configuration is a single training run.

The dense 257M model has 12 layers, width 1024, 8 attention heads, a SwiGLU MLP of width 2816, an untied output head, vocabulary size 50{,}304 and sequence length 1024. All dense runs train for 39{,}250 steps of 128 sequences ($\approx$5.1B tokens of FineWeb-Edu \citep{penedo2024fineweb}, about 20 tokens per parameter) with learning rate $10^{-3}$, weight decay 0.1, $(\beta_1,\beta_2)=(0.9,0.99)$ ($\beta_2=0.999$ for GaLore, which updates its projection every 50 steps), gradient clipping at 1, and a warmup--stable--decay schedule with 2000 warmup steps whose last 10\% decays to zero; Muon uses momentum 0.95 and five Newton--Schulz iterations. In the Monarch models every linear projection except the output head is replaced by $P_1LP_2R$ with 2 or 4 blocks; embeddings, normalization weights and the output head are dense and are updated by AdamW in all Muon runs. Weights are kept in FP32 with BF16 autocast.
The MoE model is trained with Megatron-Core \citep{shoeybi2019megatron}. It has 18 layers of width 1024 with grouped-query attention (8 query heads, 2 key--value groups, head dimension 128), QK-normalization, rotary embeddings, RMSNorm, SwiGLU, an untied output head and sequence length 2048. The first layer has a dense FFN of width 2816; each of the other 17 holds 64 routed experts of width 256 with top-8 sigmoid routing and auxiliary-loss-free expert biases, plus one shared expert of width 256, for 1.029B total and 280M active parameters. Training uses FineWeb-Edu with the GPT-2 tokenizer, a global batch of 208 sequences and 17{,}242 steps (7.34B tokens) under a warmup--stable--decay schedule with 173 warmup steps and an exponential decay to $0.1\times$ the peak over the last 3{,}448 steps. All runs use peak learning rate $1.63\times10^{-3}$, $(\beta_1,\beta_2)=(0.9,0.95)$, weight decay 0.1, gradient clipping at 1 and BF16. Muon uses Nesterov momentum 0.95 and five quintic Newton--Schulz iterations with spectral scaling; embeddings, the output head, normalization weights and routers are updated by AdamW. In the Monarch models the attention projections, the dense FFN and every routed and shared expert are replaced by Monarch factors with 2 or 4 blocks, while the routers stay dense; Monarch--Muon orthogonalizes each block separately. Factor storage is counted for the physical fused projections: QKV is $1536\times1024$, the layer-0 gate/up projection is $5632\times1024$, and each routed or shared expert has a $512\times1024$ gate/up projection and a $1024\times256$ down projection. These dimensions are divisible by both block counts, so no padding is needed. Including the untied embedding and output head, normalization weights and routers gives exactly 1{,}028{,}926{,}976 total / 280{,}243{,}712 active parameters for the dense model, 767{,}471{,}104 / 237{,}153{,}792 for two blocks, and 435{,}777{,}024 / 170{,}618{,}368 for four blocks. Active counts include all shared-expert parameters and eight of the 64 routed experts per MoE layer.

\subsection{Training-efficiency measurements}
\label{app:setup-membench}
The systems experiments measure training operations on larger reference architectures, separately from the complete quality runs in Appendix~\ref{app:setup-pretrain}. Model-size headings give the dense reference counts. Monarch uses fewer trainable parameters at these reference widths. The timings measure the practical cost of each representation, including its forward, backward, and optimizer computations.

Table~\ref{tab:memory-time} uses dense models of 257M, 834M, 1.4B, 3.5B and 6.89B parameters, Monarch--Muon with 4 blocks, microbatch 4 and random token batches, all on one NVIDIA H100 80GB. Step time is the median wall time of a full training step (forward, backward, optimizer step and gradient reset) after warmup; the optimizer step is timed separately. Peak memory is the maximum allocated memory during the step, and optimizer state is the memory of all optimizer buffers, including projection matrices. For GaLore, APOLLO, FRUGAL and Fira the projection is rebuilt every 200 steps; the rebuild is timed separately and excluded from the step time.

\FloatBarrier
\subsection{Monarch projection of Moonlight}
\label{app:setup-moonlight}
Moonlight-16B-A3B and Moonlight-A share the architecture and the training data and differ only in the optimizer; we use both after 42{,}000 training steps. The architecture has 27 layers with multi-head latent attention (query, shared key--value down-projection, key--value up-projection and output matrices), a dense FFN in layer 0, and 64 routed plus 2 shared experts in layers 1--26 (15.96B parameters). Every selected matrix is replaced by its closed-form 4-block Monarch projection, which keeps 28--50\% of its parameters. The Gaussian reference in Figure~\ref{fig:relerr} is the projection error of i.i.d.\ Gaussian matrices of the same shapes. KL divergence to the original model and NLL are measured on 32 held-out sequences of 2048 tokens from a mixture of FineWeb-Edu, code and mathematical text; MMLU \citep{hendrycks2021mmlu} is 5-shot with 100 questions per subject and GSM8K \citep{cobbe2021gsm8k} is 4-shot on 400 problems, both run with the LM Evaluation Harness \citep{gao2024evalharness}. The original models score 61.9 (Muon) and 61.3 (AdamW) on MMLU and 38.0 and 38.3 on GSM8K.

\paragraph{Weight-space alignment.}
\label{app:moonlight-conversion-analysis}
Closed-form Monarch projection~\citep{dao2022monarchexpressivestructuredmatrices} minimizes weight-reconstruction error for a fixed factorization. We apply the closed-form projection with four blocks and compare relative Frobenius errors against Gaussian matrices of the same shapes. 
Across Moonlight layers, most normalized errors are near 1 (Figure~\ref{fig:relerr}). The query projections of the AdamW-trained checkpoint are a notable exception. Thus, the optimizer alone does not generally induce alignment with the fixed Monarch structure in dense training.

\begin{figure}[!htbp]
\centering
\includegraphics[width=0.75\linewidth]{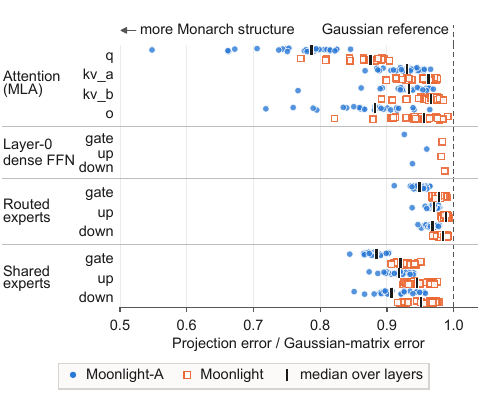}
\caption{Moonlight-16B-A3B: 4-block weight-projection errors normalized by Gaussian errors at matching shapes. Values near 1 indicate little alignment with Monarch structure. Points denote layers, pooling routed experts by projection type.}
\label{fig:relerr}
\end{figure}

\paragraph{Joint conversion.}
The sensitivity of a complete conversion also depends on interactions between layers. Replacing FFNs in layers 2--25 together produces KL divergence $7.4$--$16.7\times$ the sum of their isolated effects. Candidate sets must therefore be evaluated jointly. 

Table~\ref{tab:moonlight-joint} shows that position and component type remain important at the same storage budget: late routed-expert conversion retains substantially more MMLU~\citep{hendrycks2021mmlu} accuracy than early routed-expert conversion at identical 24.1\% parameter savings. Selecting late FFNs removes 14.5\% of all parameters while keeping MMLU within 1.1 points of the original checkpoints, without retraining. These comparisons support a concrete conversion strategy: use isolated-layer sensitivity to identify candidates, then validate their joint effect before fine-tuning.

\begin{table}[H]
\centering
\caption{Moonlight joint 4-block conversion without retraining. Saved: fraction of all parameters removed. KL: nats/token to the original model. MMLU: 5-shot accuracy (\%).}
\label{tab:moonlight-joint}
\footnotesize
\setlength{\tabcolsep}{2.2pt}
\maxsizebox{\linewidth}{!}{%
\begin{tabular}{lrrrrr}
\toprule
& & \multicolumn{2}{c}{\textbf{KL}} & \multicolumn{2}{c}{\textbf{MMLU}} \\
\cmidrule(lr){3-4}\cmidrule(l){5-6}
\textbf{Replaced blocks} & \textbf{Saved} & \textbf{Muon} & \textbf{AdamW} & \textbf{Muon} & \textbf{AdamW} \\
\midrule
None & 0 & 0 & 0 & 61.9 & 61.3 \\
FFN, layers 19--25 & 14.5\% & 0.48 & 0.33 & 60.8 & 60.4 \\
FFN, layers 14--25 & 24.8\% & 1.26 & 0.71 & 50.5 & 53.8 \\
FFN, layers 8--25 & 37.2\% & 4.62 & 1.24 & 29.1 & 32.1 \\
Routed experts, layers 14--25 & 24.1\% & 0.53 & 0.33 & 59.3 & 58.3 \\
Routed experts, layers 2--13 & 24.1\% & 1.45 & 0.78 & 24.5 & 33.2 \\
All blocks & 55.4\% & 11.12 & 7.93 & 24.5 & 25.3 \\
\bottomrule
\end{tabular}}
\end{table}

\FloatBarrier
\subsection{Fine-tuning}
\label{app:setup-finetune}
We use Qwen2.5-7B (28 layers, 7.62B parameters) with 2 blocks per converted projection; with 2 blocks a square projection keeps all of its parameters, so the savings come from the MLP and the key and value projections. Each model is fine-tuned for one epoch (313 steps) on 10{,}000 examples of the Tulu~3 SFT mixture. The metric is the negative log-likelihood of the response tokens on a held-out test split (184{,}841 tokens); the learning rate is chosen on a separate validation split (168{,}277 tokens). In Table~\ref{tab:sft-selective}, the dense, all-projection and MLP 14--26 rows use the grids $\{10^{-5},3\cdot10^{-5},10^{-4}\}$ (AdamW) and $\{3\cdot10^{-4},10^{-3},3\cdot10^{-3}\}$ (Muon); MLP 8--26 uses $\{3\cdot10^{-6},10^{-5},3\cdot10^{-5}\}$ and $\{3\cdot10^{-4},10^{-3},3\cdot10^{-3}\}$; the gate+up rows use $\{3\cdot10^{-6},10^{-5},3\cdot10^{-5}\}$ and $\{10^{-4},3\cdot10^{-4}\}$. Several selected rates lie at the lower edge of their grid (dense with both optimizers, MLP 14--26, and MLP 8--26 with Muon). Table~\ref{tab:conversion} uses $\{3\cdot10^{-6},10^{-5},3\cdot10^{-5}\}$ and $\{10^{-4},3\cdot10^{-4},10^{-3}\}$ for all rows. All results are single runs.

\paragraph{Weight-space diagnostic.}
\label{app:qwen-weight-diagnostics}
Figure~\ref{fig:relerr-qwen} extends the weight-space diagnostic to Qwen. Its 84 MLP matrices have a mean relative Frobenius error of 0.52, compared with 0.56 for Gaussian matrices under the same two-block projection. This limited alignment motivates fitting layer outputs on calibration activations rather than minimizing weight-reconstruction error alone.

\begin{figure}[!htbp]
\centering
\includegraphics[width=0.65\linewidth]{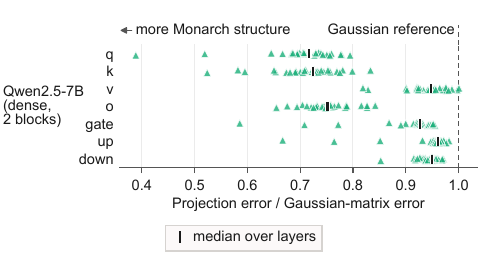}

\caption{Qwen2.5-7B: 2-block weight-projection errors normalized by Gaussian errors at matching shapes. Points denote layers.}
\label{fig:relerr-qwen}
\end{figure}

\FloatBarrier

\paragraph{Initialization at fixed factor shapes.}
Table~\ref{tab:conversion} tests output-preserving initialization for the selected MLPs of layers 14--26. The weight-space diagnostic is given in Appendix~\ref{app:qwen-weight-diagnostics}, and the initialization comparison is discussed in Section~\ref{sec:exp-finetune}. All four initializations produce Monarch factors of identical size. The rotation variant additionally uses the dense transformations specified below. (A) is the closed-form projection. (B) starts from A and alternates between the two factors to minimize $\|XW^\top - X\widehat W(L,R)^\top\|_F$, where $X$ are the inputs of the matrix in the dense model on 128 calibration examples: the left factor is solved exactly by least squares and the right one by 30 conjugate-gradient steps, for four passes, keeping the pass with the lowest error on 32 held-out examples. (C) searches, by an assignment problem (16 iterations), a permutation of the smaller dimension of each matrix (the input of gate and up, the output of down) that minimizes the weight error of the closed-form projection, and then applies B. (D) learns orthogonal rotations $Q_i$ of the MLP input, shared by the gate and up projections, and $Q_o$ of its output by a Procrustes problem, and then applies B to the rotated weights and their correspondingly rotated inputs; the rotations cost two dense $3584\times3584$ products per layer.

\FloatBarrier
\subsection{Low-precision training}
\label{app:setup-lowbit}
The two FP8 settings in Table~\ref{tab:lowbit-full} are separate runs. The GEMM setting changes the computation and activation-storage recipe, while the state setting isolates optimizer-buffer storage. Neither row represents their combination. The GEMM recipe reduces full-step time by 14\% for Monarch--Muon and 15\% for Monarch--AdamW. FP8 states reduce stored optimizer buffers by $3.7\times$, lowering the final-step peak by 0.95~GB and 1.31~GB, respectively. The dense BF16 MLP reference for the largest shape in Table~\ref{tab:lowbit-layer} takes 26.08~ms, giving $1.60\times$ and $2.22\times$ speedups for the two- and four-block FP8 implementations. These layer timings use prequantized weights and do not measure a complete training step.

All low-precision runs train the two-block Monarch model with the recipe of Appendix~\ref{app:setup-pretrain}. \emph{FP8 GEMM}: every Monarch block multiplication of both factors runs in FP8 in the forward and backward passes, with E4M3 weights and inputs, E5M2 gradients, per-tensor scaling and one batched cuBLASLt GEMM per factor; on top of this, the inputs of normalizations and activation functions are stored in FP8 with per-group scales. The output head and the attention core stay in higher precision, the residual stream is kept in BF16 rather than FP32, and validation loss is evaluated in BF16. \emph{FP8 optimizer states}: BF16 GEMMs and activations, with the Muon momentum and the AdamW moments stored in E4M3 with per-group scales and decoded to FP32 for the update. Weights and gradients are FP32 in all settings. The layer-level benchmark of Table~\ref{tab:lowbit-layer} times one RMSNorm + SwiGLU MLP block in Monarch form, forward and backward, with weights already quantized.

\section{Additional spectra}\label{app:additional-spectra}Figure~\ref{fig:monarch-gate-up-spectra} gives the compact layer-0 gate/up view used in the main text. Figures~\ref{fig:monarch-spectra-layers-a} and~\ref{fig:monarch-spectra-layers-b} provide layer-wise spectra for representative layers and all projected components.

\begin{figure}[!ht]
\centering
\includegraphics[width=0.7\textwidth]{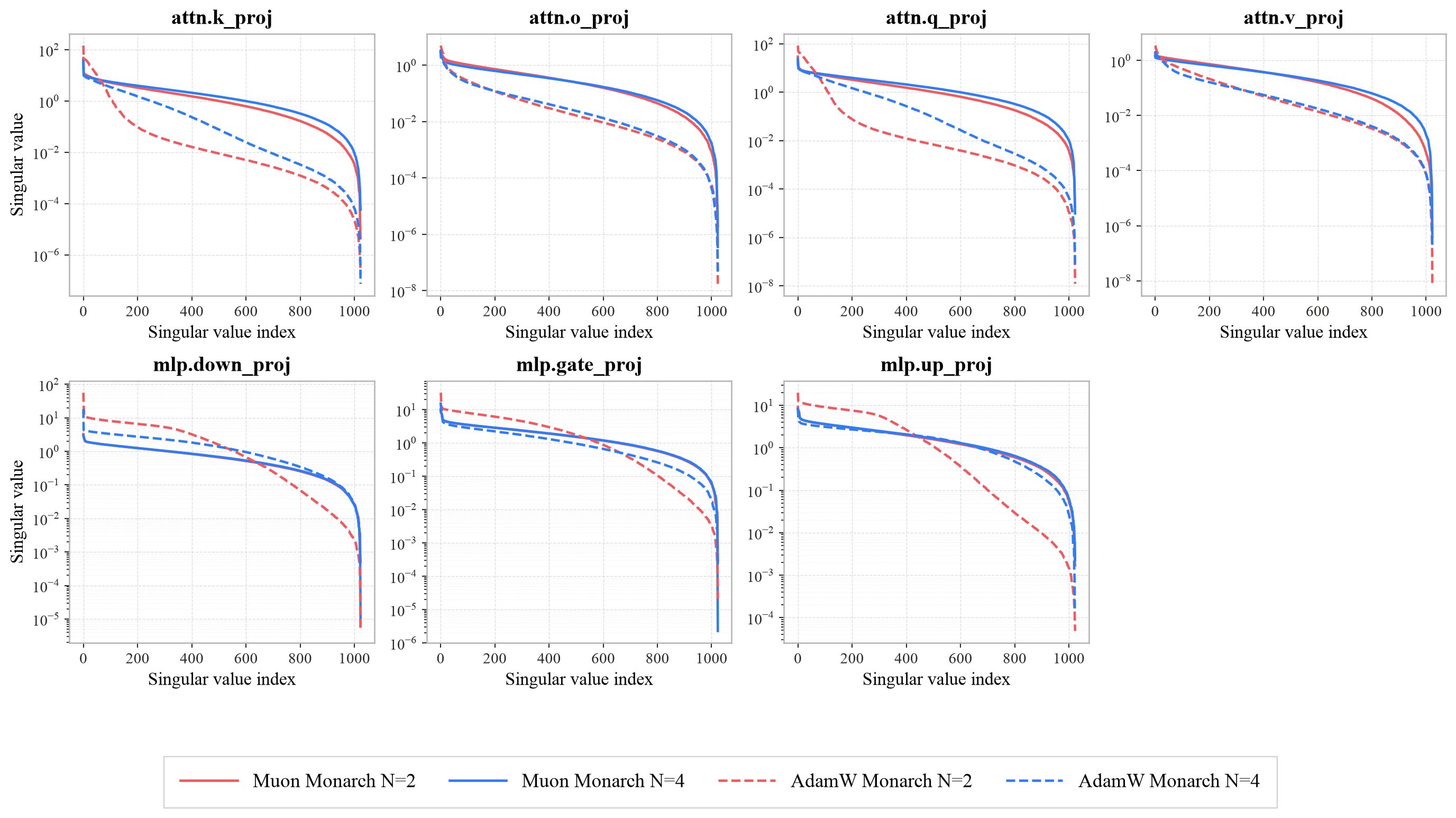}
\par\vspace{0.5em}\par
\includegraphics[width=0.7\textwidth]{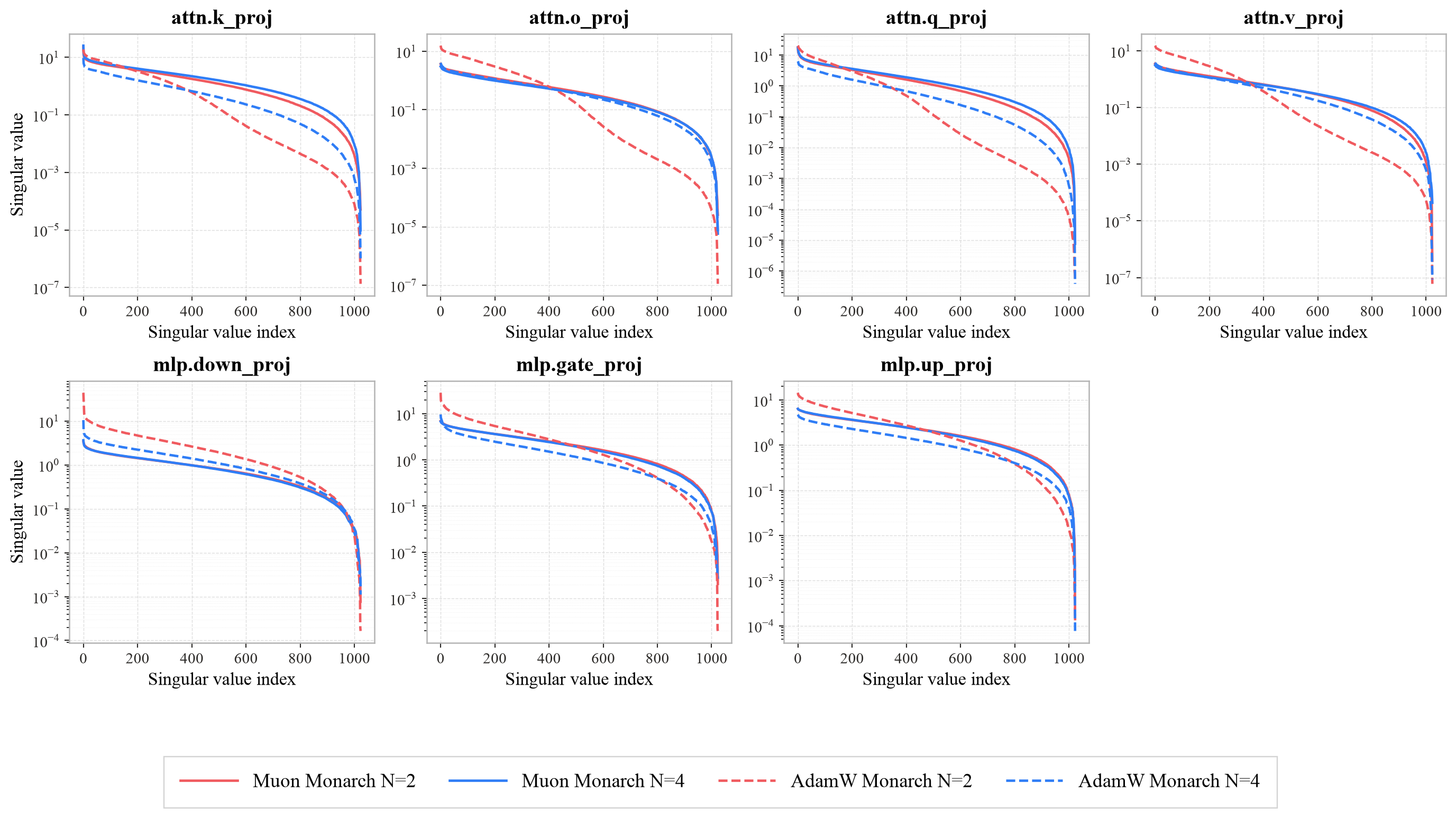}
\caption{Layer-wise singular-value spectra for layers 0 and 4 at the end of training. These plots extend Figure~\ref{fig:monarch-gate-up-spectra} from the compact gate/up illustration to all attention and MLP projections.}
\label{fig:monarch-spectra-layers-a}
\end{figure}
\begin{figure}[!p]
\centering
\includegraphics[width=0.9\textwidth]{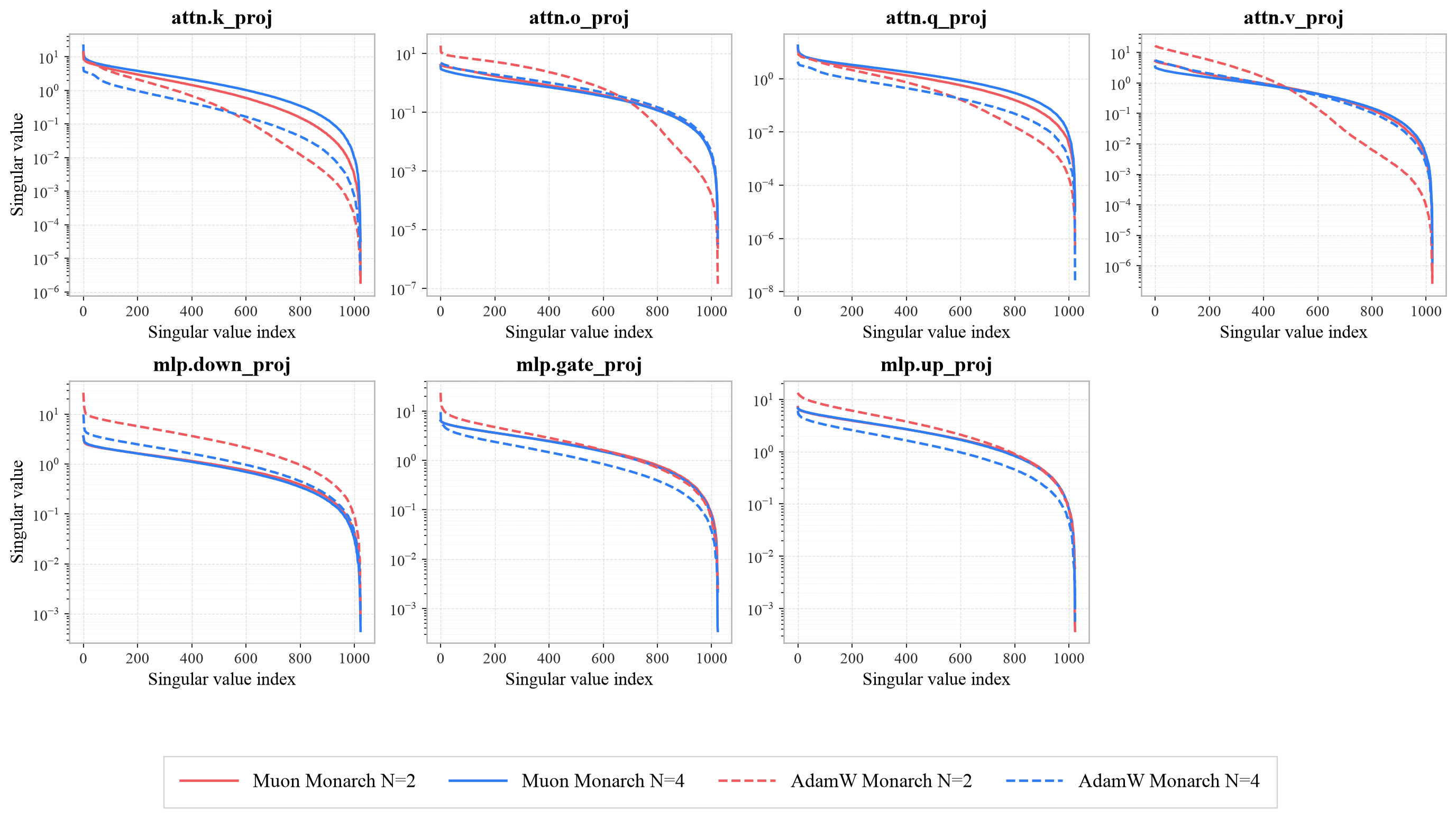}
\par\vspace{0.5em}\par
\includegraphics[width=0.9\textwidth]{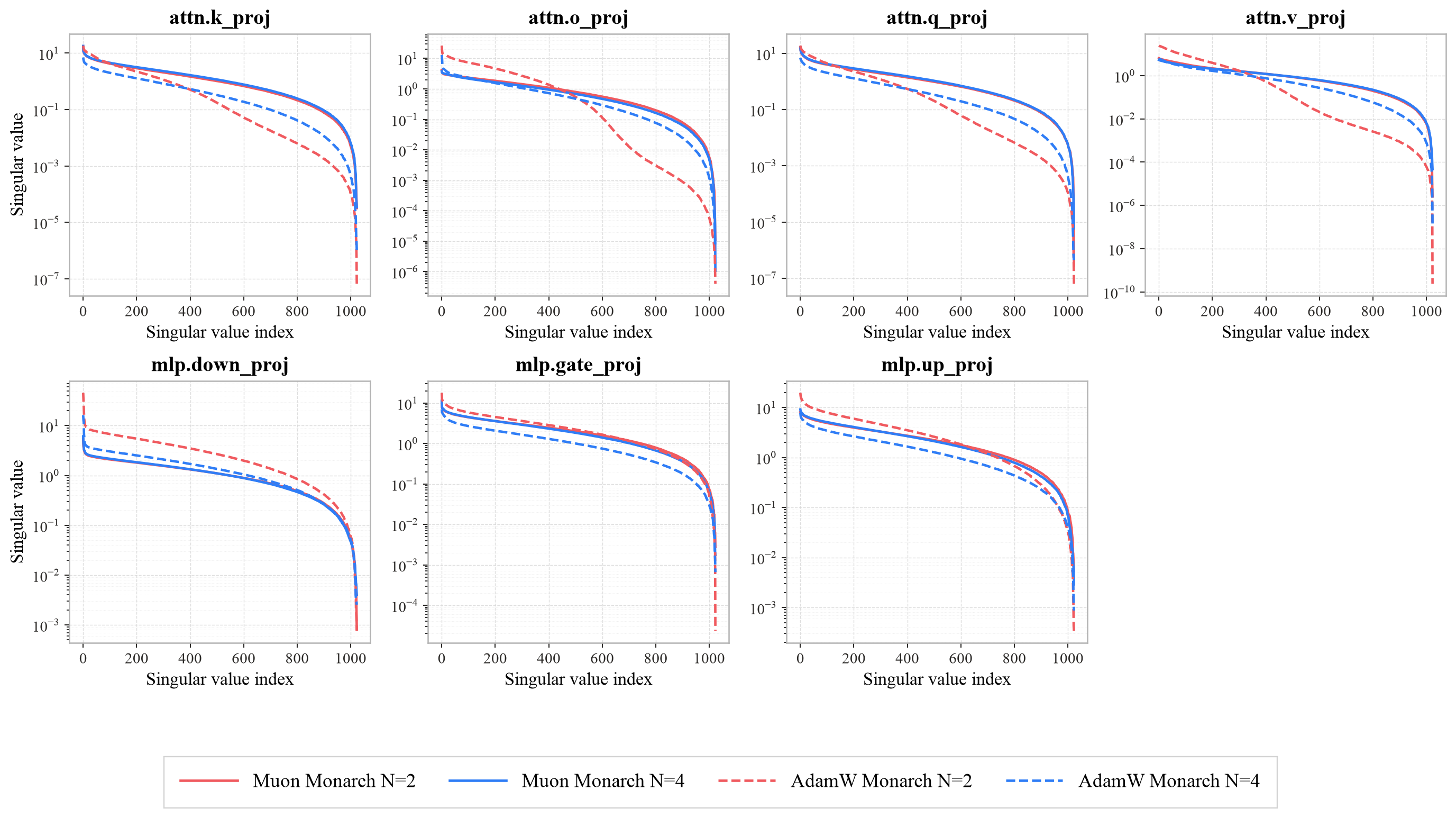}
\caption{Layer-wise singular-value spectra for layers 8 and 11 at the end of training. The same comparison shows that Monarch--Muon generally preserves a heavier spectral tail than Monarch--AdamW at fixed Monarch structure.}
\label{fig:monarch-spectra-layers-b}
\end{figure}

\clearpage
\section{Theoretical Details and Proofs}
\label{app:theory-proofs}

This appendix first establishes the block geometry, then restates the Gluon assumptions and convergence result used in our analysis. We subsequently give the complete Monarch--Muon corollary and verify its relation to the source statements.

\subsection{Block geometry and full-rank expressivity}
\label{app:blockwise-muon-step}

For the canonical square parameterization $\widehat W=PLP^\top R$, where $P$ is a permutation matrix,
$$
\det(\widehat W)=\det(PLP^\top R)
=\prod_{i=1}^m\det(L_i)\det(R_i).
$$
Consequently, nonsingular factor blocks give a nonsingular represented weight. With $m$ equal-size $d/m\times d/m$ blocks in each factor, the parameter count is $2m(d/m)^2=2d^2/m$; this is a restriction on connectivity, not an imposed rank bound.

The following lemma identifies the independent block problems solved by the exact-polar version of Algorithm~\ref{alg:monarch-muon}.

\begin{lemma}[Blockwise decomposition of the Muon subproblem]
\label{lem:blockwise-muon-step}
Let
$$
\boldsymbol M=\bdiag(M_1,\ldots,M_K),
\qquad
\boldsymbol\Delta=\bdiag(\Delta_1,\ldots,\Delta_K),
$$
and let \(\eta>0\). Then:
$$
\boxed{
\begin{aligned}
\min_{\|\boldsymbol\Delta\|_{\mathrm{op}}\leq\eta}
\langle\boldsymbol M,\boldsymbol\Delta\rangle_F
&=
\sum_{j=1}^K
\min_{\|\Delta_j\|_{\mathrm{op}}\leq\eta}
\langle M_j,\Delta_j\rangle_F
\end{aligned}
}
$$
\end{lemma}

\begin{proof}
By the definition of the operator norm:
$$
\begin{aligned}
\|\boldsymbol\Delta\|_{\mathrm{op}}^2
&=
\sup_{\|x\|_2=1}\|\boldsymbol\Delta x\|_2^2
=
\sup_{\sum_{j=1}^K\|x_j\|_2^2=1}
\sum_{j=1}^K\|\Delta_jx_j\|_2^2\\
&=
\sup_{\substack{a_j\geq0,\ \sum_j a_j=1\\
                 \|u_j\|_2=1}}
\sum_{j=1}^K a_j\|\Delta_ju_j\|_2^2\\
&=
\max_{1\leq j\leq K}
\sup_{\|u_j\|_2=1}\|\Delta_ju_j\|_2^2
=
\max_{1\leq j\leq K}\|\Delta_j\|_{\mathrm{op}}^2,
\end{aligned}
$$
where \(a_j=\|x_j\|_2^2\) and \(x_j=\sqrt{a_j}\,u_j\). Therefore:
$$
\|\boldsymbol\Delta\|_{\mathrm{op}}
=
\max_j\|\Delta_j\|_{\mathrm{op}},
$$
and hence
$$
\|\boldsymbol\Delta\|_{\mathrm{op}}\leq\eta
\quad\Longleftrightarrow\quad
\|\Delta_j\|_{\mathrm{op}}\leq\eta
\quad\text{for every }j.
$$

The Frobenius inner product also decomposes across the diagonal blocks:
$$
\langle\boldsymbol M,\boldsymbol\Delta\rangle_F
=
\sum_{j=1}^K\langle M_j,\Delta_j\rangle_F.
$$
Thus, both the constraint and the objective separate across blocks, giving
$$
\min_{\|\boldsymbol\Delta\|_{\mathrm{op}}\leq\eta}
\langle\boldsymbol M,\boldsymbol\Delta\rangle_F
=
\sum_{j=1}^K
\min_{\|\Delta_j\|_{\mathrm{op}}\leq\eta}
\langle M_j,\Delta_j\rangle_F.
$$

\end{proof}

\subsection{Assumptions and convergence result from Gluon}
\label{app:gluon-source}

We restate the relevant statements from version~1 of \citet{riabinin2025gluon}, with notation adjusted to distinguish the source parameter groups from our factor blocks. In this subsection, $Z=(Z_1,\ldots,Z_p)\in\mathcal S_1\times\cdots\times\mathcal S_p$, $h$ is the objective, and $\|\cdot\|_{(i)}$ is a norm on group $i$, with dual $\|\cdot\|_{(i),*}$. The source variance constant is denoted by $\sigma_{\mathrm G}$ here. The names in the brackets identify the original statement numbers.

\begin{assumption}[Gluon Assumption~1: generalized smoothness]
\label{ass:gluon-smoothness}
For each group $i$, there are constants $L_i^0,L_i^1\ge0$ such that, for every pair $Z,Y\in\mathcal S_1\times\cdots\times\mathcal S_p$,
\begin{equation}
\label{eq:gluon-source-smoothness}
\|\nabla_i h(Z)-\nabla_i h(Y)\|_{(i),*}
\le
\bigl(L_i^0+L_i^1\|\nabla_i h(Z)\|_{(i),*}\bigr)
\|Z_i-Y_i\|_{(i)},\qquad i=1,\ldots,p.
\end{equation}
This is the global layer-wise condition of \citet[Assumption~1, p.~5]{riabinin2025gluon}.
\end{assumption}

\begin{assumption}[Gluon Assumption~2: unbiasedness and bounded variance]
\label{ass:gluon-noise}
For every fixed $Z$, the stochastic oracle $\nabla h_\xi(Z)$ satisfies
\begin{equation}
\label{eq:gluon-source-noise}
\mathbb E_\xi[\nabla h_\xi(Z)]=\nabla h(Z),\qquad
\mathbb E_\xi\!\left[
\|\nabla_i h_\xi(Z)-\nabla_i h(Z)\|_{(i),*}^{2}\right]
\le\sigma_{\mathrm G}^{2},\qquad i=1,\ldots,p,
\end{equation}
for a finite constant $\sigma_{\mathrm G}\ge0$. This is the oracle condition in \citet[Assumption~2, p.~9]{riabinin2025gluon}; the variance is measured in the dual norm of each source parameter group.
\end{assumption}

\begin{theorem}[Gluon Theorem~8: one-group convergence rate]
\label{thm:gluon-source}
Take $p=1$ in Assumptions~\ref{ass:gluon-smoothness} and~\ref{ass:gluon-noise}. Let $h$ be continuously differentiable and bounded below, and write $\|\cdot\|$ and $\|\cdot\|_*$ for the group norm and its dual. Given fresh stochastic gradients $g^k=\nabla h_{\xi^k}(z^k)$, initialize $V^0=g^0$ and use
$$
\begin{aligned}
V^k&=\beta_k V^{k-1}+(1-\beta_k)g^k &&(k\ge1),\\
d^k&\in\operatorname*{argmin}_{\|d\|\le1}\langle V^k,d\rangle
    &&(k\ge0),\\
z^{k+1}&=z^k+t_1^k d^k &&(k\ge0),
\end{aligned}
$$
where
$$
\beta_k=1-(k+1)^{-1/2},\qquad
 t_1^k=t_1(k+1)^{-3/4},\qquad
 t_1=
\begin{cases}
\text{any fixed positive value},&L_1^1=0,\\
(12L_1^1)^{-1},&L_1^1>0.
\end{cases}
$$
Then, for every $T\ge1$,
\begin{equation}
\label{eq:gluon-source-rate}
\min_{0\le k<T}\mathbb E\|\nabla h(z^k)\|_*
=O\!\left(\frac{1+\log T}{T^{1/4}}\right).
\end{equation}
The hidden constant is independent of $T$ and may depend on the initial objective gap, $L_1^0,L_1^1,t_1,\sigma_{\mathrm G}$, and the fixed normed space.
\end{theorem}

Theorem~\ref{thm:gluon-source} is the one-group, asymptotic form of \citet[Appendix~D.2, Theorem~8, p.~33]{riabinin2025gluon}, which makes Theorem~2 of that paper explicit. Its two cases correspond to Parts~1 and~2 of the source result. We retain the convergence order rather than the numerical coefficients of the published bounds. The stochastic-sum estimate needed for our factor-space application is verified directly below using the Frobenius variance assumption.

\subsection{Complete statement for Monarch--Muon}
\label{app:monarch-convergence-statement}

Let $\theta=(X_1,\ldots,X_K)$, with $X_j\in\mathbb R^{p_j\times q_j}$, and $F(\theta)=\mathbb E_\xi f(P_1LP_2R;\xi)$. We use the norm, dual norm, and pairing in \eqref{eq:lmo-block-geometry}. For a compact singular value decomposition $A=U\Sigma V^\top$ containing only positive singular values, define $\operatorname{Polar}(A)=UV^\top$ and $\operatorname{Polar}(0)=0$.

The exact-polar version of Algorithm~\ref{alg:monarch-muon} initializes $M_j^{-1}=0$ and, for $k\ge0$ and $j=1,\ldots,K$, performs
\begin{equation}
\label{eq:lmo-block-update}
\begin{aligned}
M_j^k&=\beta_kM_j^{k-1}+(1-\beta_k)G_j^k,\\
X_j^{k+1}&=X_j^k-\eta_k\operatorname{Polar}(M_j^k).
\end{aligned}
\end{equation}
All gradients are sampled at the same tuple $\theta^k$, and each block of $L$ and $R$ has a separate momentum buffer.

\begin{corollary}[Exact-polar convergence of Monarch--Muon]
\label{thm:lmo-convergence}
Let $F$ be continuously differentiable and bounded below by $F_{\inf}$, and let $\theta^0$ be deterministic. Suppose the following conditions hold.

\smallskip
\noindent\textbf{Generalized smoothness.}
For some $\mathcal L_0,\mathcal L_1\ge0$ and all $\theta,\theta'$,
\begin{equation}
\label{eq:lmo-smoothness}
\|\nabla F(\theta)-\nabla F(\theta')\|_{\mathrm b,*}
\le\bigl(\mathcal L_0+\mathcal L_1
\|\nabla F(\theta)\|_{\mathrm b,*}\bigr)
\|\theta-\theta'\|_{\mathrm b}.
\end{equation}

\noindent\textbf{Conditional unbiasedness and Frobenius variance.}
Let $\mathcal F_k$ be the history before sampling the gradient at $\theta^k$. For a finite $\sigma\ge0$,
\begin{equation}
\label{eq:lmo-noise}
\mathbb E[G^k\mid\mathcal F_k]=\nabla F(\theta^k),\qquad
\mathbb E\!\left[\sum_{j=1}^K
\|G_j^k-\nabla_{X_j}F(\theta^k)\|_F^2
\,\middle|\,\mathcal F_k\right]\le\sigma^2.
\end{equation}

\noindent\textbf{Momentum and step sizes.}
Run \eqref{eq:lmo-block-update} with
\begin{equation}
\label{eq:lmo-schedules}
\begin{aligned}
\beta_k&=1-(k+1)^{-1/2},\quad \eta_k=\eta_0(k+1)^{-3/4},\\
\eta_0&=
\begin{cases}
\text{any fixed positive value},&\mathcal L_1=0,\\
(12\mathcal L_1)^{-1},&\mathcal L_1>0.
\end{cases}
\end{aligned}
\end{equation}
Then the conclusion of Theorem~\ref{thm:monarch-muon-main}, equivalently the bound~\eqref{eq:lmo-convergence-bound}, holds for every $T\ge1$. Its implicit constant may depend on $F(\theta^0)-F_{\inf}$, $\mathcal L_0,\mathcal L_1,\eta_0,\sigma$, and block dimensions, but not on $T$. Hence some iterate among the first $T$ has expected factor-gradient norm at most $\varepsilon$ within $\widetilde O(\varepsilon^{-4})$ iterations; the tilde suppresses logarithmic factors.
\end{corollary}

The smoothness condition is an assumption on the coupled factor objective $F(L,R)$; it is not implied by the Monarch parameterization alone. The conclusion concerns the minimum of expected gradient norms, rather than the last iterate, and uses exact polar updates with the stated momentum and step sizes.

\subsection{Proof of the Monarch--Muon convergence guarantee}
\label{app:convergence-proofs}

\begin{proof}[Proof of Theorem~\ref{thm:monarch-muon-main} and Corollary~\ref{thm:lmo-convergence}]
\textbf{1. One source group, many factor blocks.} We apply Theorem~\ref{thm:gluon-source} to the vectorized factor tuple as a \emph{single} parameter group, equipped with $\|\cdot\|_{\mathrm b}$. Thus $p=1$ in the Gluon notation, even though there are $K$ trainable factor blocks. This preserves the dependence of each factor gradient on all the factors. With $r=\sum_j\min(p_j,q_j)$, the correspondence is
\begin{center}
\begin{tabular}{ll}
\toprule
\textbf{Gluon quantity} & \textbf{Monarch--Muon specialization}\\
\midrule
Group count $p$ & $1$\\
Parameter $z$ and objective $h$ & $\operatorname{vec}(\theta)$ and $F$\\
Norm $\|\cdot\|$ and dual $\|\cdot\|_*$
 & $\|\cdot\|_{\mathrm b}$ and $\|\cdot\|_{\mathrm b,*}$\\
Smoothness constants $L_1^0,L_1^1$ & $\mathcal L_0,\mathcal L_1$\\
Variance constant $\sigma_{\mathrm G}$ & $\sqrt r\,\sigma$\\
Step-size constant $t_1$ & $\eta_0$\\
\bottomrule
\end{tabular}
\end{center}

\noindent\textbf{2. The LMO gives the block polar updates.}
Lemma~\ref{lem:blockwise-muon-step} and spectral--nuclear duality yield
\begin{equation}
\label{eq:lmo-oracle-identity}
\min_{\|H\|_{\mathrm b}\le1}\langle G,H\rangle
=\sum_{j=1}^K\min_{\|H_j\|_{\op}\le1}\langle G_j,H_j\rangle_F
=-\sum_{j=1}^K\|G_j\|_*.
\end{equation}
A minimizer is $H_j=-\operatorname{Polar}(G_j)$, including rank-deficient and zero blocks. A radius-$\eta_k$ LMO step is therefore exactly the factor update in \eqref{eq:lmo-block-update}. The momentum recursion also agrees with Gluon; in particular, $\beta_0=0$ gives $M^0=G^0$, its required initialization.

\smallskip
\noindent\textbf{3. Verification of the source smoothness assumption.}
With $p=1$, $h=F$, and $\|\cdot\|_{(1)}=\|\cdot\|_{\mathrm b}$, Assumption~\ref{ass:gluon-smoothness} reads
$$
\|\nabla F(\theta)-\nabla F(\theta')\|_{\mathrm b,*}
\le\bigl(L_1^0+L_1^1\|\nabla F(\theta)\|_{\mathrm b,*}\bigr)
\|\theta-\theta'\|_{\mathrm b}.
$$
This is precisely \eqref{eq:lmo-smoothness} after setting $L_1^0=\mathcal L_0$ and $L_1^1=\mathcal L_1$. No separate smoothness assumption on an individual factor is used.

\smallskip
\noindent\textbf{4. Verification of the source variance assumption.}
For a block tuple $E$, put $\|E\|_{\mathrm F,\oplus}^{2}=\sum_j\|E_j\|_F^{2}$. By the nuclear--Frobenius inequality and Cauchy--Schwarz,
$$
\|E\|_{\mathrm b,*}
=\sum_j\|E_j\|_*
\le\sum_j\sqrt{\min(p_j,q_j)}\,\|E_j\|_F
\le\sqrt r\,\|E\|_{\mathrm F,\oplus}.
$$
Together with \eqref{eq:lmo-noise}, this implies
$$
\mathbb E[G^k\mid\mathcal F_k]=\nabla F(\theta^k),\qquad
\mathbb E\!\left[
\|G^k-\nabla F(\theta^k)\|_{\mathrm b,*}^{2}
\,\middle|\,\mathcal F_k\right]\le r\sigma^2.
$$
Thus Assumption~\ref{ass:gluon-noise} holds along the iterates with $\sigma_{\mathrm G}=\sqrt r\,\sigma$.

The Frobenius formulation also supplies the weighted stochastic-sum estimate used in the source convergence argument. For deterministic weights $a_\tau$ and martingale differences $E^\tau=G^\tau-\nabla F(\theta^\tau)$,
\begin{equation}
\label{eq:lmo-stochastic-sum}
\begin{aligned}
\mathbb E\Bigl\|\sum_\tau a_\tau E^\tau\Bigr\|_{\mathrm b,*}
&\le\sqrt r\Bigl(\mathbb E
  \Bigl\|\sum_\tau a_\tau E^\tau\Bigr\|_{\mathrm F,\oplus}^{2}\Bigr)^{1/2}\le\sqrt r\,\sigma\Bigl(\sum_\tau a_\tau^2\Bigr)^{1/2}.
\end{aligned}
\end{equation}
Conditional unbiasedness eliminates cross terms in the Euclidean inner product underlying $\|\cdot\|_{\mathrm F,\oplus}$. This argument does not require independence of the noises in different factor blocks and does not assume a Euclidean variance identity for the block dual norm itself.

\smallskip
\noindent\textbf{5. Transfer of the convergence rate.}
The schedules in \eqref{eq:lmo-schedules} match those in Theorem~\ref{thm:gluon-source} with $t_1=\eta_0$. Part~1 of the source Theorem~8 applies when $\mathcal L_1=0$; Part~2 applies when $\mathcal L_1>0$ and $\eta_0=(12\mathcal L_1)^{-1}$. Using the noise control above, the one-group rate gives
$$
\min_{0\le k<T}\mathbb E\|\nabla F(\theta^k)\|_{\mathrm b,*}
=O\!\left((1+\log T)T^{-1/4}\right).
$$
The dual-norm identity in \eqref{eq:lmo-block-geometry} gives \eqref{eq:lmo-convergence-bound}. Solving this bound for $\varepsilon$ yields the stated $\widetilde O(\varepsilon^{-4})$ iteration complexity.
\end{proof}

\end{appendixpart}
\end{document}